\documentclass[11pt]{article}

\usepackage[margin=1in]{geometry}
\usepackage{amsmath}
\usepackage{natbib}
\usepackage{graphicx}
\usepackage{url}
\usepackage{algorithm}
\usepackage[noend]{algorithmic}
\usepackage{hyperref}
\usepackage{amsfonts,amssymb,amsthm}

\newcommand{\E}{\mathbb{E}}
\newcommand{\mcA}{\mathcal{A}}

\newcommand{\mcS}{\mathcal{S}}
\newcommand{\mcU}{\mathcal{U}}
\newcommand{\mcV}{\mathcal{V}}
\newcommand{\mcX}{\mathcal{X}}
\newcommand{\cu}{c_{\mcU}}
\newcommand{\cv}{c_{\mcV}}

\newtheorem{theorem}{Theorem}
\newtheorem{lemma}{Lemma}
\newtheorem{corollary}{Corollary}

\title{A Constant-Factor Bi-Criteria Approximation Guarantee\\for $k$-means++}
\usepackage{times}

 \author{Dennis Wei}

\begin{document}

\maketitle

\begin{abstract}
This paper studies the $k$-means++ algorithm for clustering as well as the class of $D^\ell$ sampling algorithms to which $k$-means++ belongs.  It is shown that for any constant factor $\beta > 1$, selecting $\beta k$ cluster centers by $D^\ell$ sampling yields a constant-factor approximation to the optimal clustering with $k$ centers, in expectation and without conditions on the dataset.  This result extends the previously known $O(\log k)$ guarantee for the case $\beta = 1$ to the constant-factor bi-criteria regime.  It also improves upon an existing constant-factor bi-criteria result that holds only with constant probability. 
\end{abstract}


\section{Introduction}
\label{sec:intro}

The $k$-means problem and its variants constitute one of the most popular paradigms for clustering \citep{Jain2010}.  Given a set of $n$ data points, the task is to group them into $k$ clusters, each defined by a cluster center, such that the sum of distances from points to cluster centers (raised to a power $\ell$) is minimized.  Optimal clustering in this sense is known to be NP-hard, in particular for $k$-means ($\ell = 2$) \citep{Dasgupta2008,Aloise2009,Mahajan2009,Awasthi2015} and $k$-medians ($\ell = 1$) \citep{Jain2002}.  In practice, the most widely used algorithm remains Lloyd's \citeyearpar{Lloyd1957,Lloyd1982} (often referred to as the $k$-means algorithm), which alternates between updating centers given cluster assignments and re-assigning points to clusters.  

In this paper, we study an enhancement to Lloyd's algorithm known as $k$-means++ \citep{Arthur2007} and the more general class of $D^\ell$ sampling algorithms to which $k$-means++ belongs.  These algorithms select cluster centers randomly from the given data points with probabilities proportional to their current costs.  The clustering can then be refined using Lloyd's algorithm.  $D^\ell$ sampling is attractive for two reasons:  First, it is guaranteed to yield an expected $O(\log k)$ approximation to the optimal clustering with $k$ centers \citep{Arthur2007}. 
Second, it is as simple as Lloyd's algorithm, both conceptually as well as computationally with $O(nkd)$ running time in $d$ dimensions. 

The particular focus of this paper is on the setting where an optimal $k$-clustering remains the benchmark but more than $k$ cluster centers can be sampled to improve the approximation.  Specifically, it is shown (see Theorem \ref{thm:main} and Corollary \ref{cor:main}) that for any constant factor $\beta > 1$, if $\beta k$ centers are chosen by $D^\ell$ sampling, then a constant-factor approximation to the optimal $k$-clustering is obtained.  This guarantee holds in expectation and for all datasets, like the one in \citet{Arthur2007}, and improves upon the $O(\log k)$ factor therein.  Such a result is known as a constant-factor bi-criteria approximation since both the optimal cost and the relevant degrees of freedom ($k$ in this case) are exceeded but only by constant factors.  

In the context of clustering, bi-criteria approximations can be valuable because an appropriate number of clusters $k$ is almost never known or pre-specified in practice.  Approaches to determining $k$ from the data are all ideally based on knowing how the optimal cost decreases as $k$ increases, but obtaining this optimal trade-off between cost and $k$ is NP-hard as mentioned earlier.  Alternatively, a simpler algorithm that has a constant-factor bi-criteria guarantee would ensure that the trade-off curve generated by this algorithm deviates by no more than constant factors along both axes from the optimal curve.  This may be more appealing than a deviation along the cost axis that grows with $k$.  Furthermore, if a solution with a specified number of clusters $k$ is truly required, then linear programming techniques can be used to select a $k$-subset from the $\beta k$ cluster centers while still maintaining a constant-factor approximation \citep{Aggarwal2009,Charikar2002}.

The main result in this paper differs from the constant-factor bi-criteria approximation established in \citet{Aggarwal2009} in that the latter holds only with constant probability as opposed to in expectation.  Using Markov's inequality, a constant-probability corollary can be derived from Theorem \ref{thm:main} herein, and doing so improves upon the approximation factor of \citet{Aggarwal2009} by more than a factor of $2$.  The present paper also differs from recent work on more general bi-criteria approximation of $k$-means by \citet{Makarychev2015}, which analyzes substantially more complex algorithms. 

In the next section, existing work on $D^\ell$ sampling and clustering approximations in general is reviewed in more detail.  Section \ref{sec:prelim} gives a formal statement of the problem, the $D^\ell$ sampling algorithm, and existing lemmas regarding the algorithm.  Section \ref{sec:results} states the main results of the paper and compares them to previous results.  Proofs are presented in Section \ref{sec:proofs} and the paper concludes in Section \ref{sec:concl}.

\subsection{Related Work}
\label{sec:relWork}

There is a considerable literature on approximation algorithms for $k$-means, $k$-medians, and related problems, spanning a wide range in the trade-off between tighter approximation factors and lower algorithm complexity.  At one end, exact algorithms \citep{Inaba1994} and several polynomial-time approximation schemes (PTAS) \citep{Matousek2000,Badoiu2002,delaVega2003,Har-Peled2004,Kumar2010,Chen2009,Feldman2007,Jaiswal2014} have been proposed for $k$-means and $k$-medians.  While these have polynomial running times in $n$, the dependence on $k$ and sometimes on the dimension $d$ is exponential or worse.  A simpler local search algorithm was shown to yield a $((3 + 2/p)^\ell + \epsilon)$ approximation for $k$-means ($\ell=2$) in \citet{Kanungo2004} and $k$-medians ($\ell=1$) in \citet{Arya2004}, the latter under the additional constraint that centers are chosen from a finite set.  This local search however requires a polynomial number of iterations of complexity $n^{O(p)}$, and \citet{Kanungo2004} also rely on a discretization to an $\epsilon$-approximate centroid set \citep{Matousek2000} of size $O(n \epsilon^{-d} \log(1/\epsilon))$.  Linear programming algorithms offer similar constant-factor guarantees with similar running times for $k$-medians (again the finite set variant) and the related problem of facility location \citep{Charikar2002,Jain2001}. 

In contrast to the above, this paper focuses on simpler algorithms in the $D^\ell$ sampling class, including $k$-means++.  In \citet{Arthur2007}, it was proved that $D^\ell$ sampling results in an $O(\log k)$ approximation, in expectation and for all datasets.  The current work builds upon \citet{Arthur2007} to extend the guarantee to the constant-factor bi-criteria regime.  \citet{Arthur2007} also provided a matching lower bound, exhibiting a dataset on which $k$-means++ achieves an expected $\Omega(\log k)$ approximation.

Sampling algorithms have been shown to yield improved $O(1)$ approximation factors provided that the dataset satisfies certain conditions.  Such a result was established in \citet{Ostrovsky2012} for $k$-means++ and other variants of Lloyd's algorithm under the condition that the dataset is well-suited in a sense to partitioning into $k$ clusters.  In \citet{Mettu2004}, an $O(1)$ approximation was shown for a somewhat more complicated algorithm called successive sampling with $O(n(k+\log n) + k^2 \log^2 n)$ running time, subject to a bound on the dispersion of the points.  A constant-factor approximation with slightly superlinear running time has also been obtained in the streaming setting \citep{Guha2003}.  

For $k$-means++, the $\Omega(\log k)$ lower bound in \citet{Arthur2007}, which holds in expectation, has spurred follow-on works on the question of whether $k$-means++ might guarantee a constant-factor approximation with reasonably large probability.  Negative answers were provided by \citet{Brunsch2013}, who showed that an approximation factor better than $(2/3) \log k$ cannot be achieved with probability higher than a decaying exponential in $k$, and \citet{Bhattacharya2014}, who showed that a similar statement holds even in $2$ dimensions.

In a similar direction to the one pursued in the present work, \citet{Aggarwal2009} showed that if the number of cluster centers can be increased to a constant factor times $k$, then a constant-factor approximation can be achieved with constant probability.  Specifically, they prove that using $\lceil 16 (k + \sqrt{k}) \rceil$ centers gives an approximation factor of $20$ with probability $0.03$, together with a general bi-criteria guarantee but without explicit constants.  An $O(1)$ factor was also obtained independently by \citet{Ailon2009} using more centers, of order $O(k\log k)$.  As mentioned, the result of \citet{Aggarwal2009} differs from Theorem \ref{thm:main} herein in being true with constant probability as opposed to in expectation.  Furthermore, Section \ref{sec:results:compare} shows that a constant-probability corollary of Theorem \ref{thm:main} improves significantly upon \citet{Aggarwal2009}. 

Recently, \citet{Makarychev2015} has also established constant-factor bi-criteria results for $k$-means.  Their work differs from the present paper in studying more complex algorithms.  First, similar to \citet{Kanungo2004}, \citet{Makarychev2015} reduce the $k$-means problem to an $\epsilon$-approximate, finite-set instance of $k$-medians of size $n^{O(\log(1/\epsilon) / \epsilon^2)}$.  Subsequently, linear programming and local search algorithms are considered, the latter the same as in \citet{Kanungo2004,Arya2004}, and both with polynomial complexity in the size of the $k$-medians instance.

\section{Preliminaries}
\label{sec:prelim}

\subsection{Problem Definition}

We are given $n$ points $x_1, \dots, x_n$ in a real metric space $\mcX$ with metric $D(x,y)$.  The objective is to choose $t$ cluster centers $c_1, \dots, c_t$ in $\mcX$ and assign points to the nearest cluster center to minimize the potential function 
\begin{equation}\label{eqn:potential}
\phi = \sum_{i=1}^{n} \min_{j=1,\dots,t} D(x_i, c_j)^\ell.
\end{equation}
A cluster is thus defined by the points $x_i$ assigned to a center $c_j$, where ties (multiple closest centers) are broken arbitrarily.  For a subset of points $\mcS$, define $\phi(\mcS) = \sum_{x_i\in\mcS} \min_{j=1,\dots,t} D(x_i, c_j)^\ell$ to be the contribution to the potential from $\mcS$; $\phi(x_i)$ is the contribution from a single point $x_i$. 

The exponent $\ell \geq 1$ in \eqref{eqn:potential} is regarded as a problem parameter.  Letting $\ell = 2$ and $D$ be Euclidean distance, we have what is usually known as the $k$-means problem, so-called because the optimal cluster centers are means of the points assigned to them.  The choice $\ell = 1$ is also popular and corresponds to the $k$-medians problem. 

Throughout this paper, an optimal clustering will always refer to one that minimizes \eqref{eqn:potential} over solutions with $t=k$ clusters, where $k \geq 2$ is given.  Likewise, the term optimal cluster and symbol $\mcA$ will refer to one of the $k$ clusters from this optimal solution.  The goal is to approximate the potential $\phi^\ast$ of this optimal $k$-clustering using $t=\beta k$ cluster centers for $\beta \geq 1$.  

\subsection{$D^\ell$ Sampling Algorithm}

The $D^\ell$ sampling algorithm chooses cluster centers randomly from $x_1,\dots,x_n$ with probabilities proportional to their current contributions to the potential, as detailed in Algorithm~\ref{alg:D_Sampling}.  Following \citet{Arthur2007}, the case $\ell = 2$ is referred to as the $k$-means++ algorithm and the probabilities used after the first iteration are referred to as $D^2$ weighting (hence $D^\ell$ in general).  For $t$ cluster centers, the running time of $D^\ell$ sampling is $O(ntd)$ in $d$ dimensions.

In practice, Algorithm~\ref{alg:D_Sampling} is used as an initialization to Lloyd's algorithm, which usually produces further decreases in the potential.  The analysis herein pertains only to Algorithm~\ref{alg:D_Sampling} and not to the subsequent improvement due to Lloyd's algorithm. 

\begin{algorithm}[tb]
   \caption{$D^\ell$ Sampling}
   \label{alg:D_Sampling}
\begin{algorithmic}
   \STATE {\bfseries Input:} Data points $x_1, \dots, x_n$, number of clusters $t$
   \STATE Select first cluster center $c_1$ uniformly at random from $x_1, \dots, x_n$.
   \STATE Compute $\phi(x_i)$ for $i=1,\dots,n$.
   \FOR{$j=2$ {\bfseries to} $t$}
   \STATE Select $j$th center $c_j = x_i$ with probability $\phi(x_i) / \phi$.
   \STATE Update $\phi(x_i)$ for $i=1,\dots,n$.
   \ENDFOR
\end{algorithmic}
\end{algorithm}

\subsection{Existing Lemmas Regarding $D^\ell$ Sampling}

The following lemmas synthesize results from \citet{Arthur2007} that bound the expected potential within a single optimal cluster due to selecting a center from that cluster with uniform or $D^\ell$ weighting, as in Algorithm~\ref{alg:D_Sampling}.  These lemmas define the constant $r_D^{(\ell)}$ appearing in the main results below and are also used in their proof.

\begin{lemma}\citep[Lemmas 3.1 and 5.1]{Arthur2007}\label{lem:firstCluster}
Given an optimal cluster $\mcA$, let $\phi$ be the potential resulting from selecting a first cluster center randomly from $\mcA$ with uniform weighting.  Then $\E[\phi(\mcA)] \leq r_u^{(\ell)} \phi^\ast(\mcA)$ for any $\mcA$, where 
\[
r_u^{(\ell)} = \begin{cases}
2, & \ell = 2 \text{ and } D \text{ is Euclidean},\\
2^\ell, & \text{otherwise}.
\end{cases}
\]
\end{lemma}

\begin{lemma}\citep[Lemma 3.2]{Arthur2007}
\label{lem:singleCluster}
Given an optimal cluster $\mcA$ and an initial potential $\phi$, let $\phi'$ be the potential resulting from adding a cluster center selected randomly from $\mcA$ with $D^\ell$ weighting.  Then $\E[\phi'(\mcA)] \leq r_D^{(\ell)} \phi^\ast(\mcA)$ for any $\mcA$, where $r_D^{(\ell)} = 2^\ell r_u^{(\ell)}$.
\end{lemma}
The factor of $2^\ell$ between $r_u^{(\ell)}$ and $r_D^{(\ell)}$ for general $\ell$ is explained just before Theorem 5.1 in \citet{Arthur2007}.

\section{Main Results}
\label{sec:results}

The main results of this paper are stated below in terms of the single-cluster approximation ratio 
$r_D^{(\ell)}$ defined by Lemma 
\ref{lem:singleCluster}.  Subsequently in Section \ref{sec:results:compare}, the results are discussed in the context of previous work. 

\begin{theorem}\label{thm:main}
Let $\phi$ be the potential resulting from selecting $\beta k$ cluster centers according to Algorithm~\ref{alg:D_Sampling}, where $\beta \geq 1$. 
The expected approximation ratio is then bounded as 
\begin{align*}
\frac{\E[\phi]}{\phi^\ast} &\leq r_D^{(\ell)} \left( 1 + \min\left\{ \frac{\varphi (k-2)}{(\beta-1)k+\varphi}, H_{k-1} \right\} \right) - \Theta\left(\frac{1}{n}\right), 
\end{align*}
where $\varphi = (1 + \sqrt{5}) / 2 \doteq 1.618$ is the golden ratio and $H_k = 1 + \frac{1}{2} + \dots + \frac{1}{k} \sim \log k$ is the $k$th harmonic number. 
\end{theorem}
In the proof of Theorem~\ref{thm:main} in Section \ref{sec:proofs:main}, it is shown that the $1/n$ term is indeed non-positive and can therefore be omitted, 
with negligible loss for large $n$. 

The approximation ratio bound in Theorem \ref{thm:main} is stated as a function of $k$.  The following corollary confirms that 
the theorem also implies a constant-factor bi-criteria approximation. 
\begin{corollary}\label{cor:main}
With the same definitions as in Theorem \ref{thm:main}, the expected approximation ratio is bounded as 
\[
\frac{\E[\phi]}{\phi^\ast} \leq r_D^{(\ell)} \left( 1 + \frac{\varphi}{\beta-1} \right).
\]
\end{corollary}
\begin{proof}
The minimum appearing in Theorem \ref{thm:main} is bounded from above by its first term.  This term is in turn increasing in $k$ with asymptote $\varphi / (\beta-1)$, which can therefore be taken as a $k$-independent bound. 
\end{proof}
It follows from Corollary \ref{cor:main} that a constant ``oversampling'' ratio $\beta > 1$ leads to a constant-factor approximation.  Theorem \ref{thm:main} offers a further refinement 
for finite $k$.

The bounds in Theorem \ref{thm:main} and Corollary \ref{cor:main} consist of two factors.  As $\beta$ increases, the second, parenthesized factor decreases to $1$ either exactly or approximately as $1/(\beta-1)$. 
The first factor of $r_D^{(\ell)}$ however is no smaller than $4$, and is a direct consequence of Lemma \ref{lem:singleCluster}.  Any improvement of Lemma \ref{lem:singleCluster} would therefore strengthen the approximation factors above.  This subject is briefly discussed in Section \ref{sec:concl}. 

\subsection{Comparisons to Existing Results}
\label{sec:results:compare}

A comparison of Theorem \ref{thm:main} to results in \citet{Arthur2007} is implicit in its statement since the $H_{k-1}$ term in the minimum comes directly from \citet[Theorems 3.1 and 5.1]{Arthur2007}.  For $k = 2, 3$, the first term in the minimum is smaller than $H_{k-1}$ for any $\beta \geq 1$, and hence Theorem \ref{thm:main} is always an improvement.  For $k > 3$, Theorem \ref{thm:main} improves upon \citet{Arthur2007} for $\beta$ greater than the critical value
\[
\beta_c = 1 + \frac{\phi(k-2 - H_{k-1})}{k H_{k-1}}.
\]
Numerical evaluation of $\beta_c$ shows that it reaches a maximum value of $1.204$ at $k = 22$ and then decreases back toward $1$ roughly as $1/H_{k-1}$.  It can be concluded that for any $k$, at most $20\%$ oversampling is required for Theorem \ref{thm:main} to guarantee a better approximation than \citet{Arthur2007}. 

The most closely related result to Theorem \ref{thm:main} and Corollary \ref{cor:main} is found in \citet[Theorem 1]{Aggarwal2009}.  The latter establishes a constant-factor bi-criteria approximation 
that holds with constant probability, as opposed to in expectation. Since a bound on the expectation implies a bound with constant probability via Markov's inequality, a direct comparison with \citet{Aggarwal2009} is possible. 
Specifically, for $\ell = 2$ and the $t = \lceil 16(k + \sqrt{k}) \rceil$ cluster centers assumed in \citet{Aggarwal2009}, Theorem \ref{thm:main} in the present work implies that 
\begin{align*}
\frac{\E[\phi]}{\phi^\ast} &\leq 8 \left( 1 + \min\left\{ \frac{\varphi (k-2)}{\lceil 15k + 16\sqrt{k} \rceil +\varphi}, H_{k-1} \right\} \right)\\ 
&\leq 8 \left( 1 + \frac{\varphi}{15} \right),
\end{align*}
after taking $k \to \infty$.  Then by Markov's inequality,  
\[
\frac{\phi}{\phi^\ast} \leq \frac{8}{0.97} \left( 1 + \frac{\varphi}{15} \right) \doteq 9.137
\]
with probability at least $1 - 0.97 = 0.03$ as in \citet{Aggarwal2009}.  This $9.137$ approximation factor 
is less than half the factor of $20$ in \citet{Aggarwal2009}.

Corollary \ref{cor:main} may also be compared to the results in \citet{Makarychev2015}, although it should be re-emphasized that the latter analyzes different, substantially more complex algorithms, with running time at least $n^{O(\log(1/\epsilon)/\epsilon^2)}$ for reasonably small $\epsilon$.  The main difference between Corollary \ref{cor:main} 
and the bounds in \citet{Makarychev2015} 
is the extra factor of $r_D^{(\ell)}$ since the factor of $1 + \phi/(\beta-1)$ is comparable, at least for moderate values of $\beta$ that are of practical interest. 
As discussed above and in Section \ref{sec:concl}, the factor of $r_D^{(\ell)}$ is due to 
Lemma \ref{lem:singleCluster} and is unlikely to be intrinsic to the $D^\ell$ sampling algorithm.

\section{Proofs}
\label{sec:proofs}

The overall strategy used to prove Theorem \ref{thm:main} is similar to that in \citet{Arthur2007}.  The key intermediate result is Lemma \ref{lem:key} below, which relates the potential at a later iteration in Algorithm \ref{alg:D_Sampling} to the potential at an earlier iteration.  
Section \ref{sec:proofs:key} is devoted to proving Lemma \ref{lem:key}.  Subsequently in Section \ref{sec:proofs:main}, Theorem \ref{thm:main} is proven by an application of Lemma \ref{lem:key}.

In the sequel, we say that an optimal cluster $\mcA$ is covered by a set of cluster centers if at least one of the centers lies in $\mcA$.  Otherwise $\mcA$ is uncovered.  Also define $\rho = r_D^{(\ell)} \phi^\ast$ as an abbreviation. 

\begin{lemma}\label{lem:key}
For an initial set of centers leaving $u$ optimal clusters uncovered, let $\phi$ denote the potential, $\mcU$ the union of uncovered clusters, and $\mcV$ the union of covered clusters.  Let $\phi'$ denote the potential resulting from adding $t \geq u$ centers, each selected randomly with $D^\ell$ weighting as in Algorithm \ref{alg:D_Sampling}.
Then the new potential is bounded in expectation as 
\[
\E[\phi' \mid \phi] \leq \cv(t,u) \phi(\mcV) + \cu(t,u) \rho(\mcU)
\]
for coefficients $\cv(t,u)$ and $\cu(t,u)$ that depend only on $t, u$.
This holds in particular for 
\begin{subequations}\label{eqn:key}
\begin{align}
\cv(t,u) &= 1 + \frac{\varphi u}{t-u+\varphi},\label{eqn:key_cv}\\
\cu(t,u) &= \begin{cases}
1 + \dfrac{\varphi (u-1)}{t-u+\varphi}, & u > 0,\\
0, & u = 0.
\end{cases}\label{eqn:key_cu}
\end{align}
\end{subequations}
\end{lemma}

\subsection{Proof of Lemma \ref{lem:key}}
\label{sec:proofs:key}

Lemma~\ref{lem:key} is proven using induction, showing that if it holds for $(t,u)$ and $(t,u+1)$, then it also holds for $(t+1, u+1)$, similar to the proof of \citet[Lemma 3.3]{Arthur2007}.  The proof is organized into three parts.
Section~\ref{sec:proofs:keyBaseCases} provides base cases.  In Section~\ref{sec:proofs:key(a)}, sufficient conditions on the coefficients $\cv(t,u)$, $\cu(t,u)$ are derived that allow the inductive step to be completed. 
In Section~\ref{sec:proofs:key(b)}, it is shown 
that the closed-form expressions in \eqref{eqn:key} are consistent with the base cases in Section~\ref{sec:proofs:keyBaseCases} and 
satisfy the sufficient conditions from 
Section~\ref{sec:proofs:key(a)}, thus completing the proof.  

\subsubsection{Base cases}
\label{sec:proofs:keyBaseCases}

This subsection exhibits two base cases of Lemma~\ref{lem:key}.  While the second of these base cases does not conform to the functional forms in \eqref{eqn:key}, it is shown later in Section~\ref{sec:proofs:key(b)} that the same base cases also hold with coefficients given by \eqref{eqn:key}. 

The first case corresponds to $u = 0$, for which we have $\phi(\mcV) = \phi$.  Since adding centers cannot increase the potential, i.e.\ $\phi' \leq \phi$ deterministically, Lemma~\ref{lem:key} holds with 
\begin{equation}\label{eqn:baseCase0}
\cv(t,0) = 1, \quad \cu(t,0) = 0, \quad t \geq 0.
\end{equation}

The second base case occurs for $t = u$, $u \geq 1$.  For this purpose, a slightly strengthened version of \citet[Lemma 3.3]{Arthur2007} is used, as given next.

\begin{lemma}\label{lem:AV2007Stronger}
With the same definitions as in Lemma~\ref{lem:key} except with $t \leq u$, we have 
\[
\E[\phi' \mid \phi] \leq (1 + H_t) \phi(\mcV) + (1 + H_{t-1}) \rho(\mcU) + \frac{u-t}{u} \phi(\mcU),
\]
where we define $H_{0} = 0$ and $H_{-1} = -1$ for convenience.
\end{lemma}
The improvement is in the coefficient in front of $\rho(\mcU)$, from $(1 + H_t)$ to $(1 + H_{t-1})$.  The proof follows that of \citet[Lemma 3.3]{Arthur2007} with some differences and is deferred to Appendix \ref{app:AV2007Stronger}. 

Specializing to the case $t = u$, Lemma~\ref{lem:AV2007Stronger} coincides with Lemma~\ref{lem:key} with coefficients 
\begin{equation}\label{eqn:baseCaseDiag}
\cv(u,u) = 1 + H_u, \quad \cu(u,u) = 1 + H_{u-1}.
\end{equation}

\subsubsection{Sufficient conditions on coefficients}
\label{sec:proofs:key(a)}

In this subsection, it is assumed inductively that Lemma \ref{lem:key} holds for $(t,u)$ and $(t,u+1)$.  The induction to the case $(t+1, u+1)$ is then completed under the following sufficient conditions on the coefficients:
\begin{subequations}\label{eqn:condLinear}
\begin{gather}
\cv(t,u+1) \geq 1,\label{eqn:condLinear1}\\
(\cv(t,u+1) - \cu(t,u+1)) \cv(t,u)^2 \geq (\cu(t,u+1) - \cv(t,u))^2,\label{eqn:condLinear2}
\end{gather}
\end{subequations}
and
\begin{subequations}\label{eqn:recursions}
\begin{align}
\cv(t+1, u+1) &\geq \frac{1}{2} \left[ \cv(t,u) + \left( \cv(t,u)^2 + 4\max\{\cv(t,u+1) - \cv(t,u), 0\} \right)^{1/2} \right],\label{eqn:recursion_cv}\\
\cu(t+1, u+1) &\geq \cv(t,u).\label{eqn:recursion_cu}
\end{align}
\end{subequations}
The first pair of conditions \eqref{eqn:condLinear} applies to the coefficients involved in the inductive hypothesis for $(t,u)$ and $(t,u+1)$.  The second pair \eqref{eqn:recursions} can be seen as a recursive specification of the new coefficients for $(t+1,u+1)$.  This inductive step together with base cases \eqref{eqn:baseCase0} and \eqref{eqn:baseCaseDiag} 
are sufficient to extend Lemma \ref{lem:key} to all $t > u$, starting with $(t+1,u+1) = (2,1)$ from $(t,u) = (1,0)$ and $(t,u+1) = (1,1)$.

The inductive step is broken down into a series of three lemmas, each building upon the last.  The first lemma applies the inductive hypothesis to derive a bound on the potential that depends not only on $\phi(\mcV)$ and $\rho(\mcU)$ but also on $\phi(\mcU)$.
\begin{lemma}\label{lem:inductHyp}
Assume that Lemma~\ref{lem:key} holds for $(t,u)$ and $(t,u+1)$.  Then for the case $(t+1, u+1)$, i.e.\ $\phi$ corresponding to $u+1$ uncovered clusters and $\phi'$ resulting after adding $t+1$ centers, 
\begin{multline*}
\E[\phi' \mid \phi] \leq \min\left\{ \frac{\cv(t,u) \phi(\mcU) + \cv(t,u+1) \phi(\mcV)}{\phi(\mcU) + \phi(\mcV)} \phi(\mcV) \right.\\
\left. {} + \frac{\cv(t,u) \phi(\mcU) + \cu(t,u+1) \phi(\mcV)}{\phi(\mcU) + \phi(\mcV)} \rho(\mcU), \phi(\mcU) + \phi(\mcV) \right\}.
\end{multline*}
\end{lemma}
\begin{proof}
We consider the two cases in which the first of the $t+1$ new centers is chosen from either the covered set $\mcV$ or the uncovered set $\mcU$, similar to the proof of Lemma~\ref{lem:AV2007Stronger}.  Denote by $\phi^1$ the potential after adding the first new center.

\emph{Covered case:} This case occurs with probability $\phi(\mcV) / \phi$ and leaves the covered and uncovered sets unchanged.  We then invoke Lemma~\ref{lem:key} with $(t, u+1)$ (one fewer center to add) and $\phi^1$ playing the role of $\phi$.  The contribution to $\E[\phi' \mid \phi]$ from this case is then bounded by 
\begin{align}
&\frac{\phi(\mcV)}{\phi} \left( \cv(t, u+1) \phi^1(\mcV) + \cu(t, u+1) \rho(\mcU) \right)\nonumber\\
\leq {} &\frac{\phi(\mcV)}{\phi} \left( \cv(t, u+1) \phi(\mcV) + \cu(t, u+1) \rho(\mcU) \right),\label{eqn:inductHypCov}
\end{align}
noting that $\phi^1(\mcS) \leq \phi(\mcS)$ for any set $\mcS$.

\emph{Uncovered case:} We consider each uncovered cluster $\mcA \subseteq \mcU$ separately.  With probability $\phi(\mcA) / \phi$, the first new center is selected from $\mcA$, moving $\mcA$ from the uncovered to the covered set and reducing the number of uncovered clusters by one.  Applying Lemma~\ref{lem:key} for $(t,u)$, the contribution to $\E[\phi' \mid \phi]$ is bounded by
\[
\frac{\phi(\mcA)}{\phi} \left[ \cv(t, u) \left(\phi^1(\mcV) + \phi^1(\mcA)\right) + \cu(t, u) (\rho(\mcU) - \rho(\mcA)) \right].
\]
Taking the expectation with respect to possible centers in $\mcA$ 
and using Lemma~\ref{lem:singleCluster} and $\phi^1(\mcV) \leq \phi(\mcV)$, we obtain the further bound 
\[
\frac{\phi(\mcA)}{\phi} \left[ \cv(t, u) (\phi(\mcV) + \rho(\mcA)) + \cu(t, u) (\rho(\mcU) - \rho(\mcA)) \right].
\]
Summing over $\mcA \subseteq \mcU$ yields 
\begin{align}
&\frac{\phi(\mcU)}{\phi} (\cv(t, u) \phi(\mcV) + \cu(t, u) \rho(\mcU)) + \frac{\cv(t, u) - \cu(t, u)}{\phi} \sum_{\mcA\subseteq\mcU} \phi(\mcA) \rho(\mcA)\nonumber\\
\leq {} &\frac{\phi(\mcU)}{\phi} \cv(t, u) (\phi(\mcV) + \rho(\mcU)),\label{eqn:inductHypUncov}
\end{align}
using the inner product bound \eqref{eqn:innerProdBound}.

The result follows from summing \eqref{eqn:inductHypCov} and \eqref{eqn:inductHypUncov} and combining with the trivial bound $\E[\phi' \mid \phi] \leq \phi = \phi(\mcU) + \phi(\mcV)$.
\end{proof}

As noted above, the bound in Lemma~\ref{lem:inductHyp} depends on $\phi(\mcU)$, the potential over uncovered clusters.  This quantity can be arbitrarily large or small.  In the next lemma, $\phi(\mcU)$ is eliminated by maximizing with respect to it.
\begin{lemma}\label{lem:elimU}
Assume that Lemma~\ref{lem:key} holds for $(t,u)$ and $(t,u+1)$ with $\cv(t,u+1) \geq 1$.  Then for the case $(t+1, u+1)$ in the sense of Lemma~\ref{lem:inductHyp}, 
\[
\E[\phi' \mid \phi] \leq \, \frac{1}{2} \, \cv(t,u) (\phi(\mcV) + \rho(\mcU)) + \frac{1}{2} \max\left\{ \cv(t,u) (\phi(\mcV) + \rho(\mcU)), \sqrt{Q} \right\},
\]
where
\begin{multline*}
Q = \left( \cv(t,u)^2 - 4\cv(t,u) + 4\cv(t,u+1) \right) \phi(\mcV)^2\\
+ 2\left( \cv(t,u)^2 - 2\cv(t,u) + 2\cu(t,u+1) \right) \phi(\mcV) \rho(\mcU) + \cv(t,u)^2 \rho(\mcU)^2.
\end{multline*}
\end{lemma}
\begin{proof}
The result is obtained by maximizing the bound in Lemma~\ref{lem:inductHyp} with respect to $\phi(\mcU)$.  Let $B_{1}(\phi(\mcU))$ and $B_{2}(\phi(\mcU))$ denote the two terms in the minimum.  The derivative of $B_{1}(\phi(\mcU))$ is given by 
\[
B'_{1}(\phi(\mcU)) = \frac{\phi(\mcV)}{(\phi(\mcU) + \phi(\mcV))^{2}} \bigl[ (\cv(t,u) - \cv(t,u+1)) \phi(\mcV) + (\cv(t,u) - \cu(t,u+1)) \rho(\mcU) \bigr],
\]
which does not change sign as a function of $\phi(\mcU)$.  The two cases $B'_{1}(\phi(\mcU)) \geq 0$ and $B'_{1}(\phi(\mcU)) < 0$ are considered separately below.  Taking the maximum of the resulting bounds \eqref{eqn:elimU1}, \eqref{eqn:elimU2} establishes the lemma.

\emph{Case $B'_{1}(\phi(\mcU)) \geq 0$:} Both $B_{1}(\phi(\mcU))$ and $B_{2}(\phi(\mcU))$ are non-decreasing functions of $\phi(\mcU)$.  The former has the finite supremum 
\begin{equation}\label{eqn:elimU1}
\cv(t,u) (\phi(\mcV) + \rho(\mcU)),
\end{equation}
whereas the latter increases without bound.  Therefore $B_{1}(\phi(\mcU))$ eventually becomes the smaller of the two and \eqref{eqn:elimU1} can be taken as an upper bound on $\min\{ B_{1}(\phi(\mcU)), B_{2}(\phi(\mcU)) \}$.

\emph{Case $B'_{1}(\phi(\mcU)) < 0$:} At $\phi(\mcU) = 0$, we have $B_{1}(0) = \cv(t,u+1) \phi(\mcV) + \cu(t,u+1) \rho(\mcU)$ and $B_{2}(0) = \phi(\mcV)$.  The assumption $\cv(t,u+1) \geq 1$ implies that $B_{1}(0) \geq B_{2}(0)$.  Since $B_{1}(\phi(\mcU))$ is now a decreasing function, the two functions must intersect and the point of intersection then provides an upper bound on $\min\{ B_{1}(\phi(\mcU)), B_{2}(\phi(\mcU)) \}$.  

Solving for the intersection leads after some algebra to a quadratic equation in $\phi(\mcU)$:
\begin{multline*}
0 = \phi(\mcU)^{2} + \left[ 2\phi(\mcV) - \cv(t,u) (\phi(\mcV) + \rho(\mcU)) \right] \phi(\mcU)\\
{} + \phi(\mcV) \left( \phi(\mcV) - \cv(t,u+1) \phi(\mcV) - \cu(t,u+1) \rho(\mcU) \right).
\end{multline*}
Again by the assumption $\cv(t,u+1) \geq 1$, the constant term in this quadratic equation is non-positive, implying that one of the roots is also non-positive and can be discarded.  The remaining positive root is given by 
\[
\phi(\mcU) = \frac{1}{2} \cv(t,u) (\phi(\mcV) + \rho(\mcU)) - \phi(\mcV) + \frac{1}{2} \sqrt{Q}
\]
after simplifying the discriminant to match the stated expression for $Q$.  Evaluating either $B_{1}(\phi(\mcU))$ or $B_{2}(\phi(\mcU))$ at this root gives 
\begin{equation}\label{eqn:elimU2}
\frac{1}{2} \cv(t,u) (\phi(\mcV) + \rho(\mcU)) + \frac{1}{2} \sqrt{Q}.
\end{equation}
\end{proof}

The bound in Lemma~\ref{lem:elimU} is a function of $\phi(\mcV)$ and $\rho(\mcU)$ only but is nonlinear, in contrast to the desired form in Lemma~\ref{lem:key}.  The next step is to linearize the bound by imposing additional conditions \eqref{eqn:condLinear} on the coefficients.

\begin{lemma}\label{lem:linear}
Assume that Lemma~\ref{lem:key} holds for $(t,u)$ and $(t,u+1)$ with coefficients satisfying \eqref{eqn:condLinear}.  Then for the case $(t+1, u+1)$ in the sense of Lemma~\ref{lem:inductHyp}, 
\begin{align*}
\E[\phi' \mid \phi] \leq \, \frac{1}{2} \left[ \cv(t,u) + \left( \cv(t,u)^2 + 4\max\{\cv(t,u+1) - \cv(t,u), 0\} \right)^{1/2} \right] \phi(\mcV) + \cv(t,u) \rho(\mcU).
\end{align*}
\end{lemma}
\begin{proof}
It suffices to linearize the $\sqrt{Q}$ term in Lemma~\ref{lem:elimU}.  In particular, we aim to bound the quadratic function $Q$ from above by the square $(a \phi(\mcV) + b \rho(\mcU))^2$ for all $\phi(\mcV), \rho(\mcU)$ and some choice of $a, b \geq 0$.  The cases $\phi(\mcV) = 0$ and $\rho(\mcU) = 0$ require that 
\begin{align*}
a^2 &\geq \cv(t,u)^2 + 4(\cv(t,u+1) - \cv(t,u)),\\
b^2 &\geq \cv(t,u)^2.
\end{align*}
Setting these inequalities to equalities, the remaining condition for the cross-term is 
\[
ab \geq \cv(t,u)^2 + 2(\cu(t,u+1) - \cv(t,u)).
\]
Equivalently for $a, b \geq 0$,
\begin{multline*}
a^2 b^2 = \left(\cv(t,u)^2 + 4(\cv(t,u+1) - \cv(t,u))\right) \cv(t,u)^2\\
\geq \left(\cv(t,u)^2 + 2(\cu(t,u+1) - \cv(t,u))\right)^2.
\end{multline*}
We rearrange to obtain
\begin{multline*}
4(\cv(t,u+1) - \cv(t,u)) \cv(t,u)^2\\ \geq 4\cv(t,u)^2 (\cu(t,u+1) - \cv(t,u)) + 4(\cu(t,u+1) - \cv(t,u))^2,
\end{multline*}
\[
(\cv(t,u+1) - \cu(t,u+1)) \cv(t,u)^2 \geq (\cu(t,u+1) - \cv(t,u))^2,
\]
the last of which is true by assumption \eqref{eqn:condLinear}.  Thus we conclude that 
\[
\sqrt{Q} \leq \sqrt{\cv(t,u)^2 + 4(\cv(t,u+1) - \cv(t,u))} \phi(\mcV) + \cv(t,u) \rho(\mcU).
\]
Combining this last inequality with Lemma~\ref{lem:elimU} proves the result.
\end{proof}

Given conditions \eqref{eqn:condLinear} and Lemma \ref{lem:linear}, the inductive step for Lemma~\ref{lem:key} can be completed by defining $\cv(t+1, u+1)$ and $\cu(t+1, u+1)$ recursively as in \eqref{eqn:recursions}. 
$\hfill\qed$

Equations \eqref{eqn:condLinear} and \eqref{eqn:recursions} provide sufficient conditions on the coefficients $\cv(t,u)$ and $\cu(t,u)$ to establish Lemma \ref{lem:key} by induction.  Section \ref{sec:proofs:key(b)} shows that these conditions are satisfied by \eqref{eqn:key}.  To motivate the functional form chosen in \eqref{eqn:key}, we first explore the behavior of solutions that satisfy \eqref{eqn:recursions} in particular.  This is done by treating \eqref{eqn:recursions} as a recursion, taking the inequalities to be equalities, and numerically evaluating $\cv(t+1,u+1)$ and $\cu(t+1,u+1)$ starting from the base cases \eqref{eqn:baseCase0} and \eqref{eqn:baseCaseDiag} as boundary conditions. 
More specifically, the computation is carried out as an outer loop over increasing $u$ starting from $u+1 = 1$, and an inner loop over $t$ starting from $t = u+1$.  Figure~\ref{fig:recursion_cv} plots the resulting values for $\cv(t,u)$ over the region $t \geq u$ ($\cu(t,u)$ is simply a shifted copy).  The most striking feature of Figure~\ref{fig:recursion_cv} is that the level contours appear to be lines $t \propto u$ emanating from the origin.  Sampling values at multiple points $(t,u)$ suggests that $\cv(t,u) \approx t/(t-u)$. 
The plot also has the properties that $\cv(t,u)$ is decreasing in $t$ for fixed $u$ and increasing in $u$ for fixed $t$.  These observations lead to the functional form for $\cv(t,u)$ proposed 
in Section \ref{sec:proofs:key(b)}. 

\begin{figure}[ht]
\begin{center}
\centerline{\includegraphics[width=0.65\textwidth]{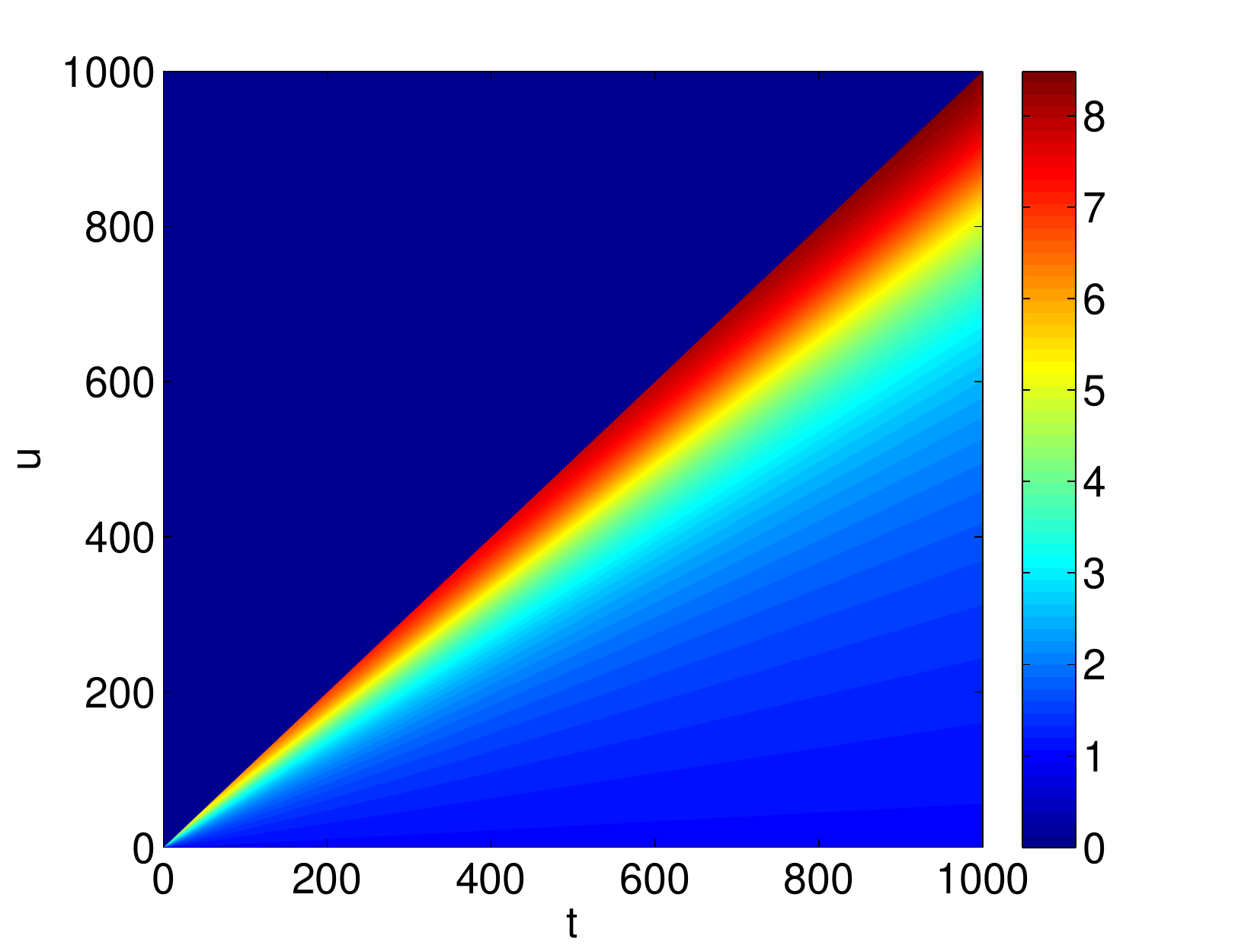}}
\caption{Coefficient $\cv(t,u)$ evaluated numerically in the region $t \geq u$ using recursion \eqref{eqn:recursion_cv} (treated as an equality) with boundary conditions \eqref{eqn:baseCase0} and \eqref{eqn:baseCaseDiag}. The numerical values approximate the function $t / (t-u)$.}
\label{fig:recursion_cv}
\end{center}
\vskip -0.2in
\end{figure}

As for conditions \eqref{eqn:condLinear}, it can be verified directly using the base cases \eqref{eqn:baseCase0} and \eqref{eqn:baseCaseDiag} that they are satisfied for $(t,u) = (1,0)$.  The subsequent 
numerical values in Figure \ref{fig:recursion_cv} 
were found to satisfy \eqref{eqn:condLinear} for all $t > u$ as well.  
This suggests that recursion~\eqref{eqn:recursions} is self-perpetuating in the sense that if \eqref{eqn:condLinear} are satisfied for $(t,u)$, then the values for $\cv(t+1, u+1)$, $\cu(t+1, u+1)$ resulting from \eqref{eqn:recursions} will also satisfy \eqref{eqn:condLinear} for $(t+1, u)$ and $(t+1, u+1)$, i.e.\ points to the right and upper-right.  This self-perpetuating property is not proven however in the present paper.  Instead, it is shown that the proposed functional form \eqref{eqn:closedForm} satisfies \eqref{eqn:condLinear} directly.

\subsubsection{Proof with specific form for coefficients}
\label{sec:proofs:key(b)}

We now prove that Lemma \ref{lem:key} holds for coefficients $\cv(t,u)$, $\cu(t,u)$ given by \eqref{eqn:closedForm} below.  These expressions are more general than \eqref{eqn:key} and are based on
the observations drawn from Figure~\ref{fig:recursion_cv}.
\begin{subequations}\label{eqn:closedForm}
\begin{align}
\cv(t,u) &= \frac{t + au + b}{t - u + b} = 1 + \frac{(a+1) u}{t - u + b}, \quad t \geq u,\label{eqn:closedForm_cv}\\
\cu(t,u) &= \begin{cases}
\cv(t-1, u-1), & t \geq u > 0,\\
0, & t \geq u = 0.
\end{cases}\label{eqn:closedForm_cu}
\end{align}
\end{subequations}
Here $a$ and $b$ are parameters introduced to add flexibility to the basic form $t / (t - u)$ suggested by Figure~\ref{fig:recursion_cv}, 
subject to the constraints $a > -1$, $b > 0$, 
\begin{subequations}\label{eqn:cond_ab}
\begin{align}
a+1 &\geq b,\label{eqn:cond_ab1}\\
ab &\geq 1.\label{eqn:cond_ab2}
\end{align}
\end{subequations}
Equation \eqref{eqn:key} is obtained at the end from \eqref{eqn:closedForm} by optimizing the parameters $a$ and $b$. 
Note that with $a + 1 > 0$, \eqref{eqn:closedForm_cv} is decreasing in $t$ for fixed $u > 0$ and increasing in $u$ for fixed $t$. 

Given the inductive approach and the results established in Sections \ref{sec:proofs:keyBaseCases} and \ref{sec:proofs:key(a)}, the proof requires the remaining steps below. 
First, it is shown that the base cases \eqref{eqn:baseCase0}, \eqref{eqn:baseCaseDiag} from Section~\ref{sec:proofs:keyBaseCases} imply that Lemma \ref{lem:key} is true for the same base cases but with $\cv(t,u)$, $\cu(t,u)$ given by \eqref{eqn:closedForm} instead.  Second, \eqref{eqn:closedForm} is shown to satisfy conditions \eqref{eqn:condLinear} for all $t > u$, thus permitting Lemma \ref{lem:linear} to be used.  Third, \eqref{eqn:closedForm} is also shown to satisfy \eqref{eqn:recursions}, which combined with Lemma \ref{lem:linear} completes the induction.  

Considering the base cases, for $u = 0$, \eqref{eqn:baseCase0} and \eqref{eqn:closedForm} coincide so there is nothing to prove.  For the case $t = u$, $u \geq 1$, Lemma \ref{lem:key} with coefficients given by \eqref{eqn:baseCaseDiag} implies the same with coefficients given by \eqref{eqn:closedForm} provided that 
\begin{multline*}
(1 + H_u) \phi(\mcV) + (1 + H_{u-1}) \rho(\mcU) \leq \left(1 + \frac{(a+1)u}{b}\right) \phi(\mcV) + \left(1 + \frac{(a+1)(u-1)}{b}\right) \rho(\mcU) \\ \forall \; \phi(\mcV), \rho(\mcU).
\end{multline*}
This in turn is ensured if the coefficients satisfy $H_u \leq (a+1) u / b$ for all $u \geq 1$.  The most stringent case is $u = 1$ and is met by assumption \eqref{eqn:cond_ab1}.

For the second step of establishing \eqref{eqn:condLinear}, it is clear that \eqref{eqn:condLinear1} is satisfied by \eqref{eqn:closedForm_cv}.  A direct calculation presented in Appendix \ref{app:condLinear} shows that \eqref{eqn:condLinear2} is also true.
\begin{lemma}\label{lem:condLinear}
Condition \eqref{eqn:condLinear2} is satisfied for all $t > u$ if $\cv(t,u)$ and $\cu(t,u)$ are given by \eqref{eqn:closedForm} and \eqref{eqn:cond_ab2} holds.
\end{lemma}

Similarly for the third step, it suffices to show that \eqref{eqn:closedForm_cv} satisfies recursion \eqref{eqn:recursion_cv} since \eqref{eqn:closedForm_cu} automatically satisfies \eqref{eqn:recursion_cu}.  A proof is provided in Appendix \ref{app:recursion}.
\begin{lemma}\label{lem:recursion}
Recursion \eqref{eqn:recursion_cv} is satisfied for all $t > u$ if $\cv(t,u)$ is given by \eqref{eqn:closedForm_cv} and \eqref{eqn:cond_ab2} holds.
\end{lemma}

Having shown that Lemma \ref{lem:key} is true for coefficients given by \eqref{eqn:closedForm} and \eqref{eqn:cond_ab}, the specific expressions in \eqref{eqn:key} are obtained by minimizing $\cv(t,u)$ in \eqref{eqn:closedForm_cv} with respect to $a$, $b$, subject to \eqref{eqn:cond_ab}.  For fixed $a$, minimizing with respect to $b$ yields $b = a + 1$ in light of \eqref{eqn:cond_ab1}, and 
\[
\cv(t,u) = 1 + \frac{(a+1)u}{t-u+(a+1)}.
\]
Minimizing with respect to $a$ then results in $a(a+1) = 1$ from \eqref{eqn:cond_ab2}.  The solution satisfying $a > -1$ is $a = \varphi - 1$ and $b = \varphi$.  $\hfill \qed$

\subsection{Proof of Theorem \ref{thm:main}}
\label{sec:proofs:main}

Denote by $n_{\mcA}$ the number of points in optimal cluster $\mcA$.  In the first iteration of Algorithm \ref{alg:D_Sampling}, the first cluster center is selected from some $\mcA$ with probability $n_{\mcA} / n$.  Conditioned on this event, Lemma \ref{lem:key} is applied with covered set $\mcV = \mcA$, $u = k-1$ uncovered clusters, and $t = \beta k - 1$ remaining cluster centers.  This bounds the final potential $\phi'$ as  
\[
\E[\phi' \mid \phi] \leq \cv(\beta k - 1,k-1) \phi(\mcA) + \cu(\beta k - 1,k-1) (\rho - \rho(\mcA))
\]
where $\cv(t,u)$, $\cu(t,u)$ are given by \eqref{eqn:key}.  Taking the expectation over possible centers in $\mcA$ and using Lemma \ref{lem:firstCluster},
\[
\E[\phi' \mid \mcA] \leq r_u^{(\ell)} \cv(\beta k - 1,k-1) \phi^\ast(\mcA) + \cu(\beta k - 1,k-1) (\rho - \rho(\mcA)).
\]
Taking the expectation over clusters $\mcA$ and recalling that $\rho = r_D^{(\ell)} \phi^\ast$, 
\begin{equation}\label{eqn:main1}
\E[\phi'] \leq r_D^{(\ell)} \cu(\beta k - 1,k-1) \phi^\ast - C \sum_{\mcA} \frac{n_{\mcA}}{n} \phi^\ast(\mcA),
\end{equation}
where
\[
C = r_D^{(\ell)} \cu(\beta k - 1,k-1) - r_u^{(\ell)} \cv(\beta k - 1,k-1).
\]

Next we aim to further bound the last term in \eqref{eqn:main1}. Using \eqref{eqn:key} and $r_D^{(\ell)} = 2^\ell r_u^{(\ell)}$ from Lemma \ref{lem:singleCluster}, 
\begin{align*}
C &= r_u^{(\ell)} \left( 2^\ell \cu(\beta k - 1,k-1) - \cv(\beta k - 1,k-1) \right)\\
&= r_u^{(\ell)} \frac{2^\ell \left((\beta-1)k + \varphi(k-1)\right) - (\beta-1+\varphi)k}{(\beta-1)k + \varphi}\\
&= r_u^{(\ell)} \frac{(2^\ell - 1)(\beta-1)k + \varphi((2^\ell - 1)(k-1) - 1)}{(\beta-1)k + \varphi}.
\end{align*}
The last expression for $C$ is seen to be non-negative for $\beta \geq 1$, $k \geq 2$, and $\ell \geq 1$.  Furthermore, since $n_{\mcA} = 1$ (a singleton cluster) implies that $\phi^\ast(\mcA) = 0$, we have 
\begin{equation}\label{eqn:main2}
\sum_{\mcA} n_{\mcA} \phi^\ast(\mcA) = \sum_{\mcA: n_{\mcA} \geq 2} n_{\mcA} \phi^\ast(\mcA) \geq 2\phi^\ast,
\end{equation}
with equality if $\phi^\ast$ is completely concentrated in clusters of size $2$.  Substituting \eqref{eqn:key_cu} and \eqref{eqn:main2} into \eqref{eqn:main1}, we obtain 
\begin{equation}\label{eqn:main3}
\frac{\E[\phi']}{\phi^\ast} \leq r_D^{(\ell)} \left( 1 + \frac{\varphi(k-2)}{(\beta-1)k + \varphi} \right) - \frac{2C}{n}.
\end{equation}

The last step is to recall \citet[Theorems 3.1 and 5.1]{Arthur2007}, which together state that 
\begin{equation}\label{eqn:main4}
\frac{\E[\phi']}{\phi^\ast} \leq r_D^{(\ell)} (1 + H_{k-1})
\end{equation}
for $\phi'$ resulting from selecting exactly $k$ cluster centers.  In fact, \eqref{eqn:main4} also holds for $\beta k $ centers, $\beta \geq 1$, since adding centers cannot increase the potential.  The proof is completed by taking the minimum of \eqref{eqn:main3} and \eqref{eqn:main4}. $\hfill\qed$

\section{Conclusion and Future Work}
\label{sec:concl}

This paper has shown that simple $D^\ell$ sampling algorithms, including $k$-means++, are guaranteed in expectation to attain a constant-factor bi-criteria approximation to an optimal clustering.  The contributions herein extend and improve upon previous results concerning $D^\ell$ sampling \citep{Arthur2007, Aggarwal2009}.

As noted in Section \ref{sec:results}, the constant $r_D^{(\ell)}$ in Theorem \ref{thm:main} and Corollary \ref{cor:main} represents an opportunity to further improve the approximation bounds.  One possibility is to tighten Lemmas 3.2 and 5.1 in \citet{Arthur2007}, which are the lemmas responsible for the $r_D^{(\ell)}$ factor.  A more significant improvement may result from considering not only the covering of optimal clusters by at least one cluster center, but also the effect of selecting more than one center from a single optimal cluster.  As the number of selected centers increases, an approximation factor analogous to $r_D^{(\ell)}$ would be expected to decrease.  Analysis of algorithms with similar simplicity to $D^\ell$ sampling is also of interest.


\bibliographystyle{plainnat}
\bibliography{k-means++}

\appendix

\section{Proof of Lemma \ref{lem:AV2007Stronger}}
\label{app:AV2007Stronger}

The proof follows the inductive proof of \citet[Lemma 3.3]{Arthur2007} with the notational changes $\mcX_u \to \mcU$, $\mcX_c \to \mcV$, and $8\phi_{\mathrm{OPT}} \to \rho$.  For brevity, only the differences are presented.     

For the first base case $t = 0$, $u > 0$, \citet{Arthur2007} already show that the lemma holds with coefficients $1 = 1 + H_0$, $0 = 1 + H_{-1}$, and $1 = (u-0)/u$.  Similarly for the second base case $t = u = 1$, \citet{Arthur2007} show that $\E[\phi' \mid \phi] \leq 2\phi(\mcV) + \rho(\mcU) = (1 + H_1) \phi(\mcV) + (1 + H_0) \rho(\mcU)$, as required for the stronger version here.

For the first ``covered'' case considered in the inductive step, the argument is the same and the upper bound on the contribution to $\E[\phi' \mid \phi]$ is changed to  
\begin{equation}\label{eqn:AV2007Cov}
\frac{\phi(\mcV)}{\phi} \left[ (1 + H_{t-1}) \phi(\mcV) + (1 + H_{t-2}) \rho(\mcU) + \frac{u-t+1}{u} \phi(\mcU) \right].
\end{equation}
For the second ``uncovered'' case, the first displayed expression in the right-hand column of \citet[page 1030]{Arthur2007} becomes (after applying the bound $\sum_{a\in\mcA} p_a \phi_a \leq \rho(\mcA)$ from Lemma~\ref{lem:singleCluster})
\[
\frac{\phi(\mcA)}{\phi} \left[ (1 + H_{t-1}) (\phi(\mcV) + \rho(\mcA)) + (1 + H_{t-2}) (\rho(\mcU) - \rho(\mcA)) + \frac{u-t}{u-1} (\phi(\mcU) - \phi(\mcA)) \right].
\]
Summing over all uncovered clusters $\mcA \subseteq \mcU$, the contribution to $\E[\phi' \mid \phi]$ is bounded from above by 
\begin{multline*}
\frac{\phi(\mcU)}{\phi} \left[ (1 + H_{t-1}) \phi(\mcV) + (1 + H_{t-2}) \rho(\mcU) + \frac{u-t}{u-1} \phi(\mcU) \right]\\
+ \frac{1}{\phi} \left[ (H_{t-1} - H_{t-2}) \sum_{\mcA\subseteq\mcU} \phi(\mcA) \rho(\mcA) - \frac{u-t}{u-1} \sum_{\mcA\subseteq\mcU} \phi(\mcA)^2 \right].
\end{multline*}
The inner product above can be bounded as 
\begin{equation}\label{eqn:innerProdBound}
\sum_{\mcA\subseteq\mcU} \phi(\mcA) \rho(\mcA) \leq \phi(\mcU) \rho(\mcU),
\end{equation}
with equality if both $\phi(\mcU)$, $\rho(\mcU)$ are completely concentrated in the same cluster $\mcA$.  The sum of squares term can be bounded using the power-mean inequality as in \citet{Arthur2007}.  Hence the contribution to $\E[\phi' \mid \phi]$ is further bounded by 
\begin{equation}\label{eqn:AV2007Uncov}
\frac{\phi(\mcU)}{\phi} \left[ (1 + H_{t-1}) \phi(\mcV) + (1 + H_{t-1}) \rho(\mcU) + \frac{u-t}{u} \phi(\mcU) \right].
\end{equation}

Summing the bounds in \eqref{eqn:AV2007Cov}, \eqref{eqn:AV2007Uncov}, we have 
\[
\E[\phi' \mid \phi] \leq (1 + H_{t-1}) \phi(\mcV) + \left( 1 + \frac{\phi(\mcV) H_{t-2} + \phi(\mcU) H_{t-1}}{\phi} \right) \rho(\mcU) + \frac{u-t}{u} \phi(\mcU) + \frac{\phi(\mcV)}{\phi} \frac{\phi(\mcU)}{u}.
\]
Recalling that $\phi = \phi(\mcV) + \phi(\mcU)$, the right-hand side is seen to be increasing in $\phi(\mcU)$.  Taking the worst case as $\phi(\mcU) \to \phi$ gives 
\begin{align*}
\E[\phi' \mid \phi] &\leq \left(1 + H_{t-1} + \frac{1}{u}\right) \phi(\mcV) + (1 + H_{t-1}) \rho(\mcU) + \frac{u-t}{u} \phi(\mcU)\\
&\leq (1 + H_{t}) \phi(\mcV) + (1 + H_{t-1}) \rho(\mcU) + \frac{u-t}{u} \phi(\mcU)
\end{align*}
since $1/u \leq 1/t$.  This completes the induction. $\hfill\qed$

\section{Proof of Lemma \ref{lem:condLinear}}
\label{app:condLinear}

Substituting \eqref{eqn:closedForm} into the left-most factor in \eqref{eqn:condLinear2},
\begin{align*}
\cv(t,u+1) - \cu(t,u+1)
&= \cv(t,u+1) - \cv(t-1,u)\\
&= \frac{(a+1)(u+1)}{t-u-1+b} - \frac{(a+1)u}{t-1-u+b}\\
&= \frac{a+1}{t-u-1+b}.
\end{align*}
Similarly on the right-hand side of \eqref{eqn:condLinear2},
\begin{align*}
\cu(t,u+1) - \cv(t,u)
&= \cv(t-1,u) - \cv(t,u)\\
&= \frac{(a+1)u}{t-1-u+b} - \frac{(a+1)u}{t-u+b}\\
&= \frac{(a+1)u}{(t-u+b)(t-u-1+b)}.
\end{align*}
Hence 
\begin{align}
&(\cv(t,u+1) - \cu(t,u+1)) \cv(t,u)^2 - (\cu(t,u+1) - \cv(t,u))^2\nonumber\\
&\qquad = \frac{a+1}{t-u-1+b} \left(1 + 2\frac{(a+1)u}{t-u+b} + \frac{(a+1)^2 u^2}{(t-u+b)^2} \right) - \frac{(a+1)^2 u^2}{(t-u+b)^2(t-u-1+b)^2}\nonumber\\
&\qquad = \frac{a+1}{t-u-1+b} \left(1 + 2\frac{(a+1)u}{t-u+b} \right) + \frac{(a+1)^2 u^2 \left[(a+1)(t-u-1+b) - 1\right]}{(t-u+b)^2(t-u-1+b)^2}.\label{eqn:condLinear3}
\end{align}
The first of the two summands in \eqref{eqn:condLinear3} is positive for $t > u \geq 0$.  The second summand 
is also non-negative as long as $(a+1)(t-u-1+b) \geq 1$.  The most stringent case occurs for $t = u+1$ and is implied by \eqref{eqn:cond_ab2}.  We conclude that 
\eqref{eqn:condLinear3} 
is positive, i.e.\ \eqref{eqn:condLinear2} holds. $\hfill\qed$

\section{Proof of Lemma \ref{lem:recursion}}
\label{app:recursion}

As noted earlier, \eqref{eqn:closedForm_cv} has the property that $\cv(t,u+1) \geq \cv(t,u)$ for all $t, u$.  Therefore \eqref{eqn:recursion_cv} is equivalent to 
\begin{equation}\label{eqn:recursion_cv2}
2\cv(t+1,u+1) - \cv(t,u) \geq \sqrt{\cv(t,u)^2 + 4(\cv(t,u+1) - \cv(t,u))}.
\end{equation}
Substituting \eqref{eqn:closedForm_cv} into the left-hand side,
\begin{align*}
2\cv(t+1,u+1) - \cv(t,u) 
&= 1 + 2\frac{(a+1)(u+1)}{t-u+b} - \frac{(a+1)u}{t-u+b}\\
&= 1 + \frac{(a+1)(u+2)}{t-u+b},
\end{align*}
which is seen to be positive for $t > u \geq 0$.  Hence \eqref{eqn:recursion_cv2} is in turn equivalent to 
\[
\left( 2\cv(t+1,u+1) - \cv(t,u) \right)^2 \geq \cv(t,u)^2 + 4(\cv(t,u+1) - \cv(t,u)).
\]
On the left-hand side,
\begin{equation}\label{eqn:recursion_cv3}
\left( 2\cv(t+1,u+1) - \cv(t,u) \right)^2 = 1 + 2\frac{(a+1)(u+2)}{t-u+b} + \frac{(a+1)^2(u+2)^2}{(t-u+b)^2}.
\end{equation}
On the right-hand side, 
\begin{align*}
\cv(t,u+1) - \cv(t,u) &= \frac{(a+1)(u+1)}{t-u-1+b} - \frac{(a+1)u}{t-u+b}\\
&= \frac{(a+1)(t+b)}{(t-u+b)(t-u-1+b)}\\
&= \frac{a+1}{t-u+b} \left(1 + \frac{u+1}{t-u-1+b}\right),
\end{align*}
\[
\cv(t,u)^2 = 1 + 2\frac{(a+1)u}{t-u+b} + \frac{(a+1)^2 u^2}{(t-u+b)^2},
\]
\begin{multline}\label{eqn:recursion_cv4}
\cv(t,u)^2 + 4(\cv(t,u+1) - \cv(t,u))\\
= 1 + 2\frac{(a+1)(u+2)}{t-u+b} + \frac{(a+1)^2 u^2}{(t-u+b)^2} + 4\frac{(a+1)(u+1)}{(t-u+b)(t-u-1+b)}.
\end{multline}
Subtracting \eqref{eqn:recursion_cv4} from \eqref{eqn:recursion_cv3} yields 
\begin{align*}
&\frac{4(a+1)^2 (u+1)}{(t-u+b)^2} - 4\frac{(a+1)(u+1)}{(t-u+b)(t-u-1+b)}\\
= {} &4 \frac{(a+1)(u+1) \left[ a(t-u-1+b) - 1 \right]}{(t-u+b)^2 (t-u-1+b)},
\end{align*}
which is non-negative provided that $a(t-u-1+b) \geq 1$.  As in the proof of Lemma \ref{lem:condLinear}, the most stringent case occurs for $t = u+1$ and is covered by \eqref{eqn:cond_ab2}.  We conclude that \eqref{eqn:recursion_cv3} is at least as large as \eqref{eqn:recursion_cv4}, i.e.\ \eqref{eqn:recursion_cv} holds.
$\hfill\qed$

\end{document}